\documentclass[preprint,12pt]{elsarticle}

\usepackage{lineno}
\usepackage{color}
\usepackage{cmll}
\usepackage{graphicx}
\usepackage{mathtext}
\usepackage{amssymb,amsmath,latexsym,amsfonts,array,epsfig}
\usepackage{mathrsfs}
\usepackage[matrix,arrow,curve]{xy}
\RequirePackage{fix-cm}

\newtheorem{definition}{Definition}[section]
\newtheorem{Th}{Theorem}[section]
\newtheorem{proof}{Proof}[section]
\newtheorem{denote}{Denotation}

\journal{Fuzzy Sets and Systems}

\begin{document}

\begin{frontmatter}



\title{Multi-Valued Cognitive Maps: Calculations with Linguistic Variables without Using Numbers}


\author{Dmitry Maximov}

\address{Trapeznikov Institute of Control Science Russian Academy of Sciences,
65 Profsoyuznaya str, Moscow}
\ead{jhanjaa@ipu.ru; dmmax@inbox.ru}

\begin{abstract}
A concept of multi-valued cognitive maps is introduced in this paper. The concept expands the fuzzy one. However, all variables and weights are not linearly ordered in the concept, but are only partially-ordered. Such an approach allows us to operate in cognitive maps with partially-ordered linguistic variables directly, without vague fuzzification/defuzzification methods. Hence, we may consider more subtle differences in degrees of experts' uncertainty, than in the fuzzy case. We prove the convergence of such cognitive maps and give a simple computational example which demonstrates using such a partially-ordered uncertainty degree scale.
\end{abstract}




\end{frontmatter}


\section{Introduction}
\label{intro}
The fuzzy cognitive map (FCM) concept was introduced by B. Kosko in \cite{kosko}. FCMs are considered as feedback models of causality, in which fuzzy values are assigned to concepts and causal relationships amongst them. An increase in the value of a concept implies a corresponding positive or negative increase in values of other concepts connected to it, according to the relationships. The concepts are also called nodes, and the relationships are called weights. Thus, we obtain a network similar to a neural network in which all the variables and weights take values in the interval $[0, 1]$.

Fuzzy cognitive maps have been studied and used in various fields of engineering and hard sciences \cite{glykas}. Their role is especially important in investigations of the behavior of complex dynamic systems \cite{hagi}, \cite{koskod}, \cite{gould}, \cite{craiger}. This is due to the fact that human knowledge uncertainty affects the systems definition and processing \cite{faulin}. However, fuzzy modelling of uncertainty is rather poor: the theory operates with only linearly-ordered experts' valuations, which in reality can be unordered: e.g., ``yes and no'' and ``neither yes nor no''.

The main contribution and the novelty of this paper is that we use a lattice (i.e., a partially-ordered set) in cognitive maps as the scale of experts' valuations (i.e., weights) and as the set of variables (i.e., concept values), instead of a linearly-ordered set. Thus, we may consider more subtle differences in degrees of experts' uncertainty, than in the fuzzy case. The approach continues the line of investigations in which a system state is estimated, not by numbers, but by various objects (sets, graphs, images, etc.) making up different lattices: \cite{Maximov_17}, \cite{Maximov_R}, \cite{Maximov_an}, \cite{maximov19}.  Also, exactly this concept was used in the research dedicated to the related area of multi-valued neural networks: \cite{maximov20e}, \cite{Maximov_neuro}, \cite{Maximov_neuro2}. Similar to these papers, we call such cognitive maps multi-valued ones (MVCM's).

Thus, all the variables and weights are partially-ordered linguistic ones here, and we do not use numbers in the cognitive maps' calculations\footnote{Let us note, we do not consider the problems of comparing different expert opinions and calculating their mean: we need a theory of multi-valued numbers (i.e., the lattice subsets), as an analog of fuzzy ones, to do this. The decision of the problem is left for the future.}. Nevertheless, such maps converge, and we consider the conditions of convergence (Sec. \ref{converge}). In Sec. \ref{learn} we represent a learning algorithm for weights, applicable when we know the desired range for output values. In Sec. \ref{model} we consider a simple computational model of a hybrid energy system. In Sec, \ref{disc} we discuss our experimental results and compare them with the previous ones. We give the necessary definitions used in the text in Sec. \ref{back}, and we conclude the paper in Sec. \ref{concl}.

\section{Backgrounds}
\label{back}

\subsection{Lattices \cite{birk}}
\begin{definition}
A \textbf{lattice} is a partially-ordered set having, for any two elements, their exact upper bound or join $\vee$ (sup, max) and the exact lower bound or meet $\wedge$ (inf, min).
\end{definition}
\begin{definition}
The \textbf{exact upper bound} of a subset $X\subseteq P$ of a partially-ordered set $P$ is the smallest $P$-element $a$, larger than all the elements of $X$: $min(a)\in P :\;a\geq x,\;\forall x\in P$.
\end{definition}
\begin{definition}
The \textbf{exact lower bound} is dually defined as the largest $P$-element, smaller than all the elements of $X$.
\end{definition}
\begin{definition}
A \textbf{complete lattice} is a lattice in which any two subsets have a join and a meet.
This means that in a non-empty complete lattice there is the largest ``$\top$'' and the smallest ``0'' elements.
\end{definition}
If we take such a lattice as a scale of truth values in a multi-valued logic,
then the largest element will correspond to complete truth (true), the smallest to complete falsehood (false), and intermediate elements will correspond to partial truth in the same way as  the elements of the segment [0,1] evaluate partial truth in fuzzy logic.

In logics, with such a scale of truth values, implication can be determined by multiplying lattice elements, or internally, only from lattice operations.
\begin{definition}
Lattice elements, from which all the others are obtained
by join and meet operations are called \textbf{generators} of the lattice.
\end{definition}
\begin{definition}
A lattice is called \textbf{atomic} if every two of its generators have null meets.
\end{definition}
\begin{definition}
A \textbf{Brouwer lattice} is a lattice that has internal implications.
\end{definition}
\begin{definition}
In such a lattice, the \textbf{implication} $c = a\Rightarrow b$ is defined as the largest $c:\;a\wedge b = a\wedge c$.
\end{definition}
Distribution laws for join and meet are satisfied in Brouwer lattices. The converse is true only for finite lattices.

\subsection{Residuated Lattices \cite{resid}}
In non-distributive lattices, the implication cannot be defined. However, we may introduce a multiplication of the lattice elements and use it to define an external implication.
\begin{definition}
A \textbf{residuated lattice} is an algebra $(L, \vee, \wedge, \cdot, 1, \rightarrow, \leftarrow)$ satisfying the following conditions:
\begin{itemize}
\item $(L, \vee, \wedge)$ is a lattice;\\
\item $(L, \cdot, 1)$ is a monoid; \\
\item $(\rightarrow, \leftarrow)$ is a pair of residuals of the operation $\cdot$, that means
\begin{equation*}
\forall x, y \in L: x\cdot y \leqslant z \Leftrightarrow y\leqslant x\rightarrow z \Leftrightarrow x\leqslant z\leftarrow y
\end{equation*}
\end{itemize}
In this case, the operation $\cdot$ is order preserving in each argument and for all $a, b \in L$ both the sets $\{y\in L| a\cdot y \leqslant b\}$ and $\{x\in L| x\cdot a \leqslant b\}$ each contains a greatest element ($a\rightarrow b$ and $b\leftarrow a$ respectively).
\end{definition}
The monoid multiplication $\cdot$ is distributive over $\vee$:
\begin{equation*}
x\cdot (y\vee z) = (x\cdot y)\vee(x\cdot z).
\end{equation*}
Also, $x\cdot 0 = 0\cdot x = 0$. A special case of residuated lattices is a Heyting algebra, when the monoid multiplication coincides with $\wedge$.

In non-commutative monoids, residuals $\rightarrow$ and $\leftarrow$ can be understood as having a temporal quality: $x\cdot y \leqslant z$ means ``$x \;then\; y \;entails\; z$,'' $y \leqslant x\rightarrow z$ means ``$y \;estimates\; the\; transition \;had\; x\; then\; z$,'' and $x \leqslant z\leftarrow y$ means ``$x\; estimates\; the\; opportunity$ \emph{if-ever} $y\; then\; z$.'' You may think about $x, y$, and $z$ as $bet,\; win$, and $rich$ correspondingly (Wikipedia).

\begin{definition}
A residuated lattice A is said to be \textbf{integrally closed} if it satisfies the equations
$x\cdot y \leqslant x \Longrightarrow y \leqslant 1$ and $y\cdot x \leqslant x \Longrightarrow y \leqslant 1$, or equivalently, the equations $x\rightarrow x = 1$ and $x\leftarrow x = 1$ \\ \cite{integr}.
\end{definition}
Any upper or lower bounded integrally closed residuated lattice $L$ is \textbf{integral}, i.e., $a\leqslant 1,\;\forall a\in L$ \cite{integr}.

\section{Cognitive Map Convergence}
\label{converge}
A graphical representation of an example of a (multi-valued) cognitive map (MVCM) is depicted in Fig. \ref{fig1}.
\begin{figure}
  \includegraphics[scale=0.8]{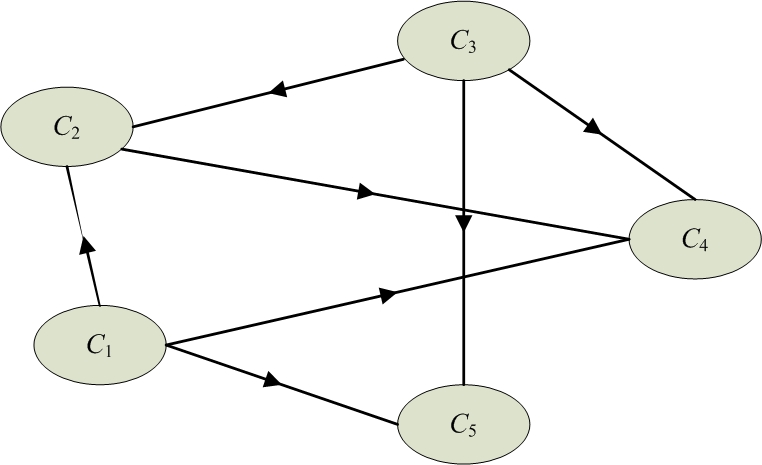}
\caption{An example of a cognitive map}
\label{fig1}       
\end{figure}
The expert knowledge on the behaviour of the system is stored in the structure
of such a graphical representation of the map. Each concept $C_{i}$ represents a characteristic of the system under consideration. It may represent goals, events, actions, states, etc. of the system. Each $C_{i}$ is characterized by an element $A_{i}$ of a lattice $L$, which represents the value of the concept, and it is obtained from an expert opinion about the real value of the systems' variable representing this concept.  Causality between concepts allows degrees of causality, which also belong to the lattice $L$; thus, the weights $w_{ij}$'s of the connections are the lattice elements and represent the expert uncertainty degrees of the concepts' mutual influences. The value of $w_{ij}$ indicates how strongly concept $C_{i}$ influences concept $C_{j}$. A simple example of such a lattice of experts' opinions and uncertainty degrees is depicted in Fig. \ref{fig2}.

\begin{figure}
  \includegraphics[scale=0.8]{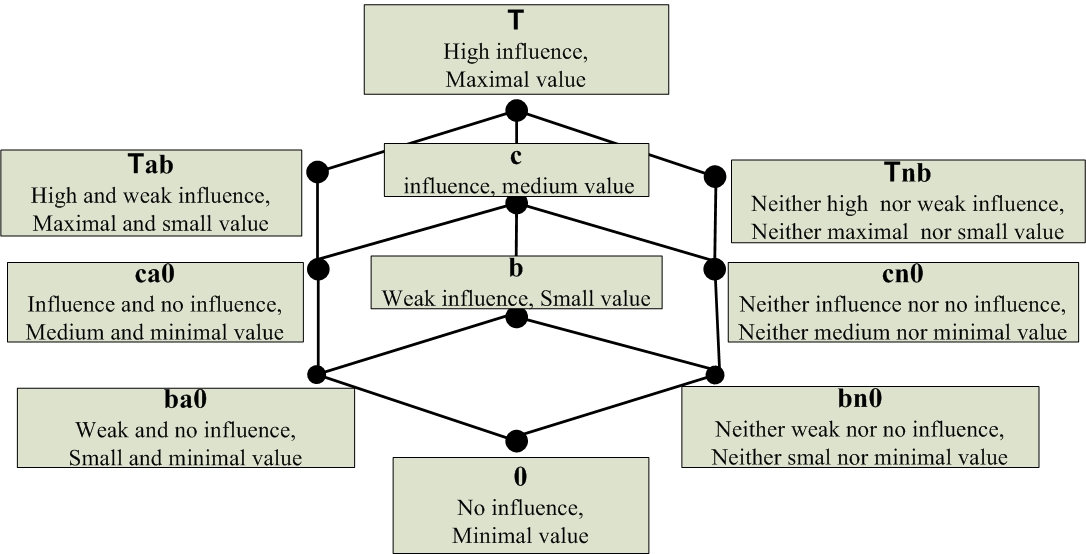}
\caption{An example of an uncertainty degree lattice}
\label{fig2}       
\end{figure}

The equation that calculates the values of concepts of FCM's with $n$ nodes, can be written in its general form as:
\begin{equation}\label{eq1}
A_{i}^{k} = f(\sum_{j=1,j\neq i}^{n}w_{ji}A_{j}^{k-1} + d_{ii}A_{i}^{k-1}).
\end{equation}
Here $A_{i}^{k}$ is the value of the concept $C_{i}$ at discrete time $k$, and $d_{ii}$ is a value of self-feedback to node $i$. All values belong to the interval $[0, 1]$, and the function $f$ normalizes its argument up to this interval. Existence and uniqueness of solutions of (\ref{eq1}) in FCM's are proved in \cite{kottas} for some such trimming functions $f$'s.

We use the equation (\ref{eq1}) for MVCM's in the following form:
\begin{equation}\label{eq2}
A_{i}^{k} =c_{i}^{k-1}\cdot f_{i}^{k-1}\cdot\bigvee_{j=1}^{n}w_{ji}\cdot A_{j}^{k-1}
\end{equation}
where all quantities take values in a residuated \emph{atomic} lattice $L$, and we use the monoid multiplication and  the lattice join, instead of the sum and numeric multiplication. We restrict ourselves to atomic lattices, since we prove the maps' convergence only in this case. Thus, the example of the simple possible lattice in Fig. \ref{fig2} is out of our consideration, and we consider a more complicated variant in our modelling example.

We do not need the quantities $f_{i}^{k}$'s and $c_{i}^{k}$'s in (\ref{eq2}) to norm the joins, since all joins are inside the lattice. Thus, we use picking up of these quantities' values in order to provide the map convergence. Although $f_{i}$'s and $c_{i}$'s do not normalize \textbf{wA}, they play the role of adjusting function: for each \textbf{wA} we get corresponding $f_{i}$ and $c_{i}$ values.

\begin{Th}\label{th}
Multi-Valued Cognitive Maps determined by equation (\ref{eq2}) where concepts and weights take values in a finite atomic residuated and integrally-closed lattice $L$ (hence, $L$ is integral), converge under a suitable choice of $c_{i}^{k}$'s and $f_{i}^{k}$'s.
\end{Th}
\begin{proof}

\begin{denote}
We denote the set of generators of a lattice element $w$ by $\{w\}$ and  the matrix of such sets by $\{w\}_{ij}$. The matrix elements are the sets of generators of, e.g., the weight matrix elements.
\end{denote}
Matrices $w_{ij}$ and $\{w\}_{ij}$ are one-to-one correspondent to each other in atomic lattices.
\begin{denote}
A \textbf{minus} sign will denote the difference operation
\begin{equation}\label{eqmin}
\{A\} - \{B\} = (\{A\} \cup \{B\}) \ominus (\{A\} \cap \{B\}),
\end{equation}
where $\ominus$ is the \textbf{set difference}\footnote{Given set $A$ and set $B$, the set difference of set $B$ from set $A$ is the set of all element in $A$, but not in $B$.}.
\end{denote}
For convenience, in what follows, we will omit the curly braces, thus, we use $A^{1} - A^{0}$ instead of $\{A^{1}\} - \{A^{0}\}$.

The map (\ref{eq2}) converges, if
\begin{equation}\label{eq4}
A_{i}^{k+1} - A_{i}^{k} \subset A_{i}^{k} - A_{i}^{k-1},
\end{equation}
since the lattice $L$ is bounded below by $\emptyset$. The process of calculating $A_{i}$ stops when $A_{i}^{k+1} = A_{i}^{k} = A_{i}^{k-1}$ for all $i$. Let us consider the following sequence of inequalities:
$$\begin{aligned}
A_{i}^{3} - A_{i}^{2} \subset A_{i}^{2} - A_{i}^{1} \subset A_{i}^{1} - A_{i}^{0};\\
A_{i}^{3} \cup A_{i}^{2} \ominus A_{i}^{3} \cap A_{i}^{2} \subset A_{i}^{2} \cup A_{i}^{1} \ominus A_{i}^{2} \cap A_{i}^{1} \subset  A_{i}^{1} \cup A_{i}^{0} \ominus A_{i}^{1} \cap A_{i}^{0};\\
\left\{\begin{aligned} A_{i}^{3} \cup A_{i}^{2} \subseteq  A_{i}^{2} \cup A_{i}^{1} \subseteq A_{i}^{1} \cup A_{i}^{0},\\
A_{i}^{3} \cap A_{i}^{2} \supseteq  A_{i}^{2} \cap A_{i}^{1} \supseteq A_{i}^{1} \cap A_{i}^{0},\\
and\; equalities\; cannot\; occur\; simultaneously.
\end{aligned}
\right.
\end{aligned}$$
Hence,
\begin{eqnarray}\label{eq5}
A_{i}^{2} \subseteq A_{i}^{1} \cup A_{i}^{0}\nonumber\\
A_{i}^{3} \subseteq A_{i}^{2} \cup A_{i}^{1},\nonumber\\
...\nonumber\\
A_{i}^{k}\subseteq A_{i}^{k-1}\cup A_{i}^{k-2};
\end{eqnarray}
\begin{equation}\label{eq6}
A_{i}^{k}\supseteq A_{i}^{k-1}\cap A_{i}^{k-2},
\end{equation}
and
equalities cannot occur simultaneously up to the end of the process
in order to satisfy (\ref{eq4}).
Substituting (\ref{eq2}) into \ref{eq5}, and passing to the lattice notation, we get:
\begin{gather}
c_{i}^{1}\cdot f_{i}^{1}\cdot\bigvee_{j=1}^{n}w_{ji}\cdot (c_{j}^{0}\cdot f_{j}^{0}\cdot\bigvee_{l=1}^{n}w_{lj}\cdot A_{l}^{0}) \leqslant (c_{i}^{0}\cdot f_{i}^{0}\cdot\bigvee_{j=1}^{n}w_{ji}\cdot A_{j}^{0}) \vee A_{i}^{0};
\end{gather}
\begin{multline}c_{i}^{2}\cdot f_{i}^{2}\cdot\bigvee_{j=1}^{n}w_{ji}\cdot A_{j}^{2} \leqslant c_{i}^{1}\cdot f_{i}^{1}\cdot\bigvee_{i=1}^{n}w_{ji}\cdot (c_{j}^{0}\cdot f_{j}^{0}\cdot\bigvee_{l=1}^{n}w_{lj}\cdot A_{l}^{0}) \vee \\ \vee c_{i}^{0}\cdot f_{i}^{0}\cdot\bigvee_{j=1}^{n}w_{ji}\cdot A_{j}^{0}.
\end{multline}
Since all $c_{i} \leqslant 1$ in the integral lattice, we may define $f_{i}^{k}$'s as right residuals\footnote{Since, $a\cdot b \leqslant c$ means for a maximal $a$: $a = c\leftarrow b$}:
\begin{multline}
f_{i}^{1} = [A_{i}^{1} \vee A_{i}^{0}]\leftarrow (\bigvee_{j=1}^{n}w_{ji}\cdot A_{j}^{1}) = \\ = [c_{i}^{0}\cdot f_{i}^{0}\cdot\bigvee_{j=1}^{n}w_{ji}\cdot A_{j}^{0} \vee A_{i}^{0}]\leftarrow (\bigvee_{j=1}^{n}w_{ji}\cdot c_{j}^{0}\cdot f_{j}^{0}\cdot\bigvee_{l=1}^{n}w_{lj}\cdot A_{l}^{0});
\end{multline}
\begin{multline}
f_{i}^{2} = [A_{i}^{2} \vee A_{i}^{1}]\leftarrow (\bigvee_{j=1}^{n}w_{ji}\cdot A_{j}^{2}) = \\ =[c_{i}^{1}\cdot f_{i}^{1}\cdot(\bigvee_{i=1}^{n}w_{ji}\cdot c_{j}^{0}\cdot f_{j}^{0}\cdot\bigvee_{l=1}^{n}w_{lj}\cdot A_{l}^{0}) \vee c_{i}^{0}\cdot f_{i}^{0}\cdot\bigvee_{j=1}^{n}w_{ji}\cdot A_{j}^{0}]\leftarrow \\ \leftarrow (\bigvee_{j=1}^{n}w_{ji}\cdot (c_{j}^{1}\cdot f_{j}^{1}\cdot\bigvee_{l=1}^{n}w_{lj}\cdot (c_{l}^{0}\cdot f_{l}^{0}\cdot\bigvee_{m=1}^{n}w_{ml}\cdot A_{m}^{0}))).
\end{multline}
...
\begin{multline}\label{eq11}
\shoveright{f_{i}^{k} = [A_{i}^{k} \vee A_{i}^{k-1}]\leftarrow (\bigvee_{j=1}^{n}w_{ji}\cdot A_{j}^{k})}
\end{multline}
By recursion, we can obtain all the $f_{i}^{k}$'s so as to satisfy the expressions (\ref{eq5}).

Also, $f_{i}^{k}$'s and $c_{i}^{k}$'s must satisfy (\ref{eq6}) in order to satisfy (\ref{eq4}). However, (\ref{eq6}) holds for such chosen $f_{i}^{k}$'s. Indeed, it must be $A_{i}^{k+1} = c_{i}^{k}\cdot f_{i}^{k}\cdot\bigvee_{i=1}^{n}w_{ji}A_{j}^{k} \geqslant A_{i}^{k}\wedge A_{i}^{k-1}$.

Let us denote
\begin{equation}
r_{i}^{k} = [A_{i}^{k}\wedge A_{i}^{k-1}]\leftarrow \bigvee_{i=1}^{n}w_{ji}\cdot A_{j}^{k}.
\end{equation}
Then, it should be $c_{i}^{k}\cdot f_{i}^{k} \geqslant r_{i}^{k}$ and $c_{i}^{k}\cdot f_{i}^{k} = r_{i}^{k}$ only if $A_{i}^{k+1} = A_{i}^{k}\wedge A_{i}^{k-1}$. In this case, $r_{i}^{k}\cdot\bigvee_{i=1}^{n}w_{ji}\cdot A_{j}^{k} = A_{i}^{k}\wedge A_{i}^{k-1}$. Since, \begin{equation}
f_{i}^{k}\cdot\bigvee_{i=1}^{n}w_{ji}\cdot A_{j}^{k} = \{[A_{i}^{k}\vee A_{i}^{k-1}]\leftarrow (\bigvee_{j=1}^{n}w_{ji}\cdot A_{j}^{k})\}\cdot\bigvee_{i=1}^{n}w_{ji}\cdot A_{j}^{k}\leqslant [A_{i}^{k}\vee A_{i}^{k-1}]
\end{equation}
and
\begin{equation}
r_{i}^{k}\cdot\bigvee_{i=1}^{n}w_{ji}\cdot A_{j}^{k} \leqslant [A_{i}^{k}\wedge A_{i}^{k-1}],
\end{equation}
where such $f_{i}^{k}$ and $r_{i}^{k}$ are maximal at satisfying the inequalities, we obtain $f_{i}^{k}\geqslant r_{i}^{k}$, since the lattice is integrally-closed. We may restrict from below $c_{i}^{k} \geqslant r_{i}^{k}\leftarrow f_{i}^{k}$ where it may be $c_{i}^{k} = r_{i}^{k}\leftarrow f_{i}^{k}$ only if $[r_{i}^{k}\leftarrow f_{i}^{k}]\cdot f_{i}^{k} = r_{i}^{k}$.

It must be $c_{i}^{k} > r_{i}^{k}\leftarrow f_{i}^{k}$ if $f_{i}^{k}\cdot\bigvee_{i=1}^{n}w_{ji}\cdot A_{j}^{k} = A_{i}^{k}\vee A_{i}^{k-1}$ in (\ref{eq5}) in order to avoid the simultaneous equality in (\ref{eq5}) and (\ref{eq6}). However, the simultaneity is not possible up to the end of the process, because, in this case, $r_{i}^{k}\leftarrow f_{i}^{k}\neq 1$, since \begin{multline}
f_{i}^{k}\cdot\bigvee_{i=1}^{n}w_{ji}\cdot A_{j}^{k} = A_{i}^{k}\vee A_{i}^{k-1} > A_{i}^{k}\wedge A_{i}^{k-1}\geqslant r_{i}^{k}\cdot\bigvee_{i=1}^{n}w_{ji}\cdot A_{j}^{k}.
\end{multline}
Hence, $f_{i}^{k} > r_{i}^{k}$ up to the end of the process. Only at the end of the process, $f_{i}^{k} = r_{i}^{k}$ and $c_{i}^{k} = 1$. Therefore, the process converges if $c_{i}^{k} \geqslant r_{i}^{k}\leftarrow f_{i}^{k}$.

Thus, we can always choose $c_{i}^{k} = 1$ in (\ref{eq2}), and MVCM's converge.
\end{proof}
However, the decision of (\ref{eq2}) is not unique: different sets of generators of initial node values may lead to different final values (unlike metric spaces under certain conditions \cite{kottas}) due to the fact that the lattice used is not linearly ordered. We may use a learning algorithm if such a situation is undesirable, and we know the required set of possible output values.

\section{Learning Weight Values}\label{learn}
We are based on ideas of \cite{papag} when constructing the required algorithm. However, we do not need two criterions to evaluate the final stage, due to the convergence of the multi-valued cognitive map. Also, the learning algorithm may be applied only to the final node values, not at every step (as in \cite{papag}); this is again due to the convergence. Finally, the algorithm is synchronous, unlike  \cite{papag}. Hence, we do not need additional expert suppositions about the firing sequence.

Let us consider the following expressions:
\begin{multline}
A_{i}^{k} = f_{i}^{k-1}\cdot\bigvee_{i=1}^{n}[(w_{ji}^{k-1} \oplus \triangle w_{ji}^{k-1})\cdot (A_{j}^{k-1} \oplus \triangle A_{j}^{k-1})] = \\
= f_{i}^{k-1}\cdot\bigvee_{i=1}^{n}(w_{ji}^{k-1}\cdot A_{j}^{k-1}) \oplus \triangle A_{i}^{k};
\end{multline}
\begin{equation}\label{eq17}
\triangle A_{i}^{k} = f_{i}^{k-1}\cdot\bigvee_{i=1}^{n}(\triangle w_{ji}^{k-1}\cdot A_{j}^{k-1}) \oplus F(w^{k-2}, A^{k-2}, \triangle w^{k-2}, \triangle A^{k-2}).
\end{equation}
Here, we have introduced a change in weights and concepts in order to pick up their values such that the output concepts $A_{i}^{k}$ would find themselves in demanding output regions $\{doc_{j}\}_{i}$\footnote{$doc_{ji}$ means $j$'s possible value of desired output concept set of $A_{i}^{k}$}. Such output sets of lattice $L$ elements should be established by experts for given initial values. The operation $\oplus$ means the join or difference (see below) depending on what you need: increase $A_{i}^{k}$ or decrease it. The term $F(w^{k-2}, A^{k-2}, \triangle w^{k-2}, \triangle A^{k-2})$ in (\ref{eq17}) should not be considered, since, such a term is absent at the first step, and all $A_{j}^{k-1}$'s and $w_{ji}^{k-1}$ are already calculated at the $k$'s step.

Let us consider the difference $gen_{ri}^{k} = doc_{ri} - A_{i}^{k}$ of a concept with one of its desired output values. The result is the number of generators in the symmetric difference (\ref{eqmin}) of generator sets of $r$'s desired output concept for $A_{i}^{k}$ and $A_{i}^{k}$ at the $k$'s step. The number should be added to, or deducted from, the concept generator number in order for the concept to become equal to the desired value. Hence,
\begin{equation}
\triangle A_{i}^{k} \leqslant gen_{ri}^{k}.
\end{equation}
Thus, we obtain the following expression using (\ref{eq17}):
\begin{equation}\label{eq19}
\triangle w_{ji}^{k-1} = (f_{i}^{k-1}\rightarrow gen_{ri})\leftarrow A_{j}^{k-1}.
\end{equation}
Then, we should deduct $\triangle w$ from $w$ if $A_{i}^{k} > doc_{ri}$: $w_{ji}^{k} = w_{ji}^{k-1} \ominus \triangle w_{ji}^{k-1}$ (let us note that the set difference operation $\ominus$ is used  here, not the symmetric difference (\ref{eqmin}), since we need to decrease exactly the number of generators). Otherwise, we take their join: $w_{ji}^{k} = w_{ji}^{k-1} \vee \triangle w_{ji}^{k-1}$.

In the case of incomparable $A_{i}^{k}$ and $doc_{ri}$, the concept first increases under the algorithm work, up to it becoming greater than the desired value, and, after that, it decreases as is described above.

If we have negative weight values, the sequence is inverse: we join $\triangle w$ and $w$ if $A_{i}^{k} > doc_{ri}$ and take the difference otherwise.

Such a comparison may be made with all elements of the desired output set, in order to choose the most suitable learned weight matrix. This algorithm describes the weight correction at every step. However, it is not necessary: we may check the condition of hitting the required region at the end of the process of firing the map, since the process converges. Thus, we may calculate (\ref{eq19}) only once at an iteration cycle if the condition is not satisfied. Such a calculation can be more suitable (e.g., it changes the initial weights less) if some concepts get the desired values only at the end of the process. In this case, they should not be corrected with the others. Naturally, changing concepts at every step can give another result in this case (see Sec. \ref{disc})

\section{Modelling Hybrid Energy Systems}
\label{model}
We use the problem formulation of a modelling example from \cite{groump}, without, however, a feedback to natural concepts. Indeed, it is hard to understand how the energy system functioning can influence sun insolation or wind. Thus, we consider the example of a Hybrid Energy System combining wind and photovoltaic subsystems; its cognitive map model is depicted in Fig. \ref{fig1}.

The model includes the following five concepts:
\begin{itemize}
\item C1: sun insolation;
\item C2: environment temperature;
\item C3: wind;
\item C4: PV-subsystem;
\item C5: Wind-Turbine-subsystem.
\end{itemize}
In this model there are two energy source decision concepts (outputs), i.e., the two energy sources are considered:
the C4: PV-subsystem and the C5: Wind-Turbine-subsystem. Concepts C1--C3 of nature and technical
factors influence the subsystems and determine how each energy source will function in this model. The concept's initial values can be obtained from experts' assessments of measurements, which take values in a lattice of linguistic variables. The experts' assessments of concept influence take values also in this lattice. Detailed information for hybrid renewable energy
systems is given in \cite{gould}, \cite{damg}. The case study from the literature
was examined in \cite{groump}, and we consider it here using a multi-valued scale for weights and values in the cognitive map (\ref{eq2}) where all $c_{i} = 1$ due to Theorem \ref{th}.

We use the bi-lattice $L$ that is built from two lattices the $L_{1}$ and the $L_{2}$  (Fig. \ref{fig3}), \ref{fig4}) by direct multiplication $\times$, as the scale of experts' assessments.
\begin{figure}
  \includegraphics[scale=0.9]{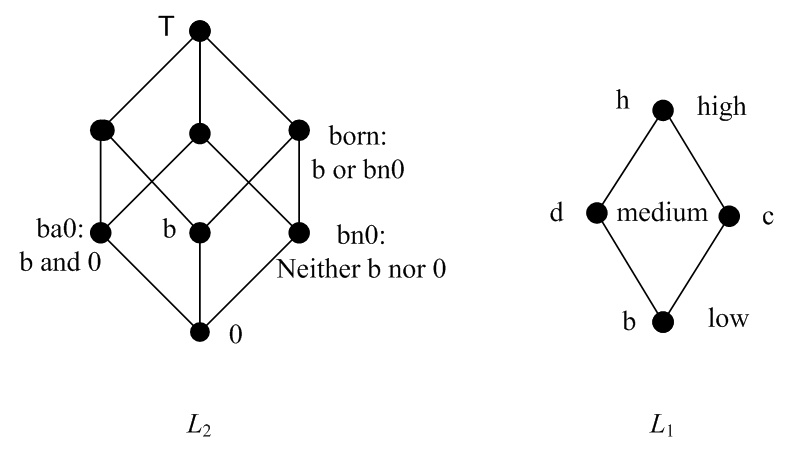}
\caption{Lattices $L_{1}$ and $L_{2}$}
\label{fig3}       
\end{figure}

\begin{figure}
  \includegraphics[scale=0.8]{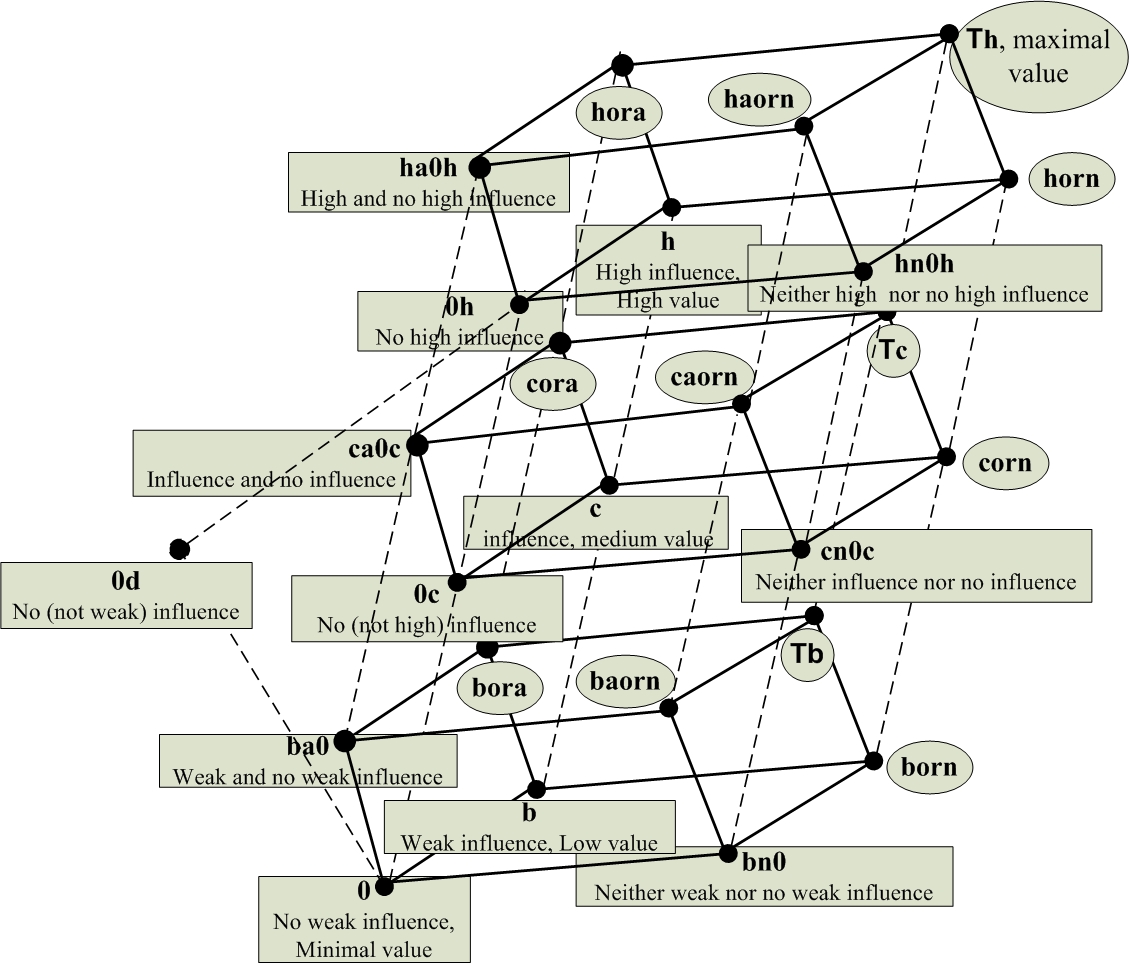}
\caption{The b-c-h path in the uncertainty degree lattice $L = L_{2}\times L_{1}$ which is used in the modelling example.
 The d-branch values are the same as in the c-branch, by replacing c to d in the names}
\label{fig4}       
\end{figure}
However, we consider the bi-lattice as the lattice where the unique partial order is generated by atoms $ba0, b, bn0, 0c$, and $0d$ Fig. \ref{fig4}. We use two linearly-ordered branches in the lattice $L_{1}$ in order to  regularly obtain the distributive and atomic lattice $L$. Then, we may use the meet $\wedge$ as the monoid multiplication. In this case, both residuals are equal and coincide with the lattice implication $\Rightarrow$. The nodes $c$ and $d$ in the $L_{1}$, corresponding to the medium value, may be interpreted as ``not high'' and ``not weak'' uncertainty degrees (this difference is not reflected in the interpretations in Fig. \ref{fig4}). There is not a universal method to build a monoid for a residuated lattice, and one should use some heuristics to determine the multiplication. Thus, we do not consider a general residual construction here and propose to investigate variants to do it in future.

It is possible to use the $L_{1}$ built from only one linearly-ordered branch. Then, the $L$ would be built by a quasi-direct multiplication to be an atomic lattice: the node $0h$ must be connected with $0$ directly. However, we consider the lattice $L = L_{2}\times L_{1}$ for more variety.

A variant of the lattice $L_{2}$ was proposed in \cite{Maximov_Ax} as the interpretation of N. A. Vasil'ev's logic ideas. Vasil'ev has suggested three types of statement: positive, negative and indifferent, instead of only positive and negative. He considered also intermediate types as a hesitation between these main ones. Similarly, we consider here three main uncertainty degrees: $ba0, b, bn0$ (Fig. \ref{fig3}) --- some assessment ``$b$'', the assessment ``$bn0$: Neither $b$ nor 0'', and the estimation ``$ba0$: $b$ and 0'' --- and the same at the levels $c - d$ and $h$ Fig. \ref{fig4}. The estimation, e.g., ``$born$: $b$ or $bn0$'' is the join of $b$ and $bn0$ and can be considered as the hesitation between $b$ and ``Neither $b$ nor 0''. Similarly, ``$bora$: $b$ or $ba0$'' is the join of $b$ and $ba0$ and can be considered as the hesitation between $b$ and ``$b$ and 0'' and so on. Thus, we obtain many different variants of uncertainty degrees in assessments, available to experts.

The connections between the concepts of the cognitive map are represented in Table \ref{tab1}, and the initial concept values are determined in Table \ref{tab2}. We tried to more or less match the data from \cite{groump}.
\begin{table}
\caption{Weights in the cognitive map for Hybrid Energy Systems}
\label{tab1}
\begin{tabular}
{c|c|c|c|c|c}
 & $C_{1}$ & $C_{2}$ & $C_{3}$ & $C_{4}$ & $C_{5}$ \\ \hline
$C_{1}$ & Th & born & 0 & hora & b  \\ \hline
$C_{2}$ & 0 & 0 &  0 & Tb & 0  \\ \hline   
$C_{3}$ & 0 & ca0c & Th & bora & hn0h  \\ \hline   
$C_{4}$ & 0 & 0  & 0  & 0 &  0  \\ \hline
$C_{5}$ & 0 & 0  & 0  & 0 &  0
\end{tabular}
\end{table}
\begin{table}
\caption{Initial concept values}
\label{tab2}
\begin{tabular}
{c|c|c|c}
  & Case 1 & Case 2 & Case 3\\ \hline
$C_{1}$ & horn  & horn & horn\\ \hline
$C_{2}$ & c  &  c &  b \\ \hline
$C_{3}$ & c  &  d &  d \\ \hline
$C_{4}$ & c  &  h &  h \\ \hline
$C_{5}$ & 0c  &  caorn &  c
\end{tabular}
\end{table}
In Case 1, all concept values are concentrated in one $L$ branch (b-c-h). In Cases 2 and 3, one initial concept value belongs to the b-d-h branch of the $L$.

Thus, we calculate the map concepts' values by (\ref{eq2}) in the following form:
\begin{equation}\label{eq2d}
C_{i}^{k} = f_{i}^{k-1}\wedge\bigvee_{j=1}^{5}(w_{ji}\wedge C_{j}^{k-1})
\end{equation}
where coefficients are calculated by (\ref{eq11}) in the following form:
\begin{equation}\label{eq11d}
{f_{i}^{k} = [\bigvee_{j=1}^{5}(w_{ji}\wedge C_{j}^{k})]\Rightarrow [C_{i}^{k} \vee C_{i}^{k-1}]}.
\end{equation}
The nodes are triggered
simultaneously, and their values interact with
ones to be updated through this process of interaction in the same iteration step. We take values $f_{i}^{0} = Th$. Hence, natural concepts the $C_{1}$ and the $C_{3}$ do not change unlike (\ref{eq1}) (taken from \cite{groump}) where they are changed due to the sigmoid transfer function $f$. However, such a choice of the $f_{i}^{0}$ is not necessary: these values may be chosen arbitrarily.

Here, the lattice of weights and values satisfies the conditions of Theorem \ref{th} where the lattice top element $Th$ coincides with the monoid neutral element, and the monoid multiplication coincides with the lattice meet.
The calculation results are represented in Tables \ref{tab3} -- \ref{tab4}.
\begin{table}
\caption{Iterations of concept values, Case 1}
\label{tab3}
  \includegraphics[scale=2]{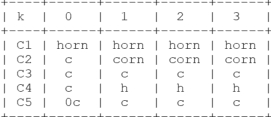}
\end{table}

\begin{table}
\caption{Iterations of concept values, Case 2 and Case 3}
\label{tab4}
\includegraphics[scale=2]{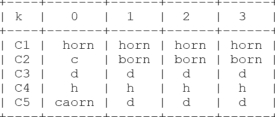} \\
\includegraphics[scale=2]{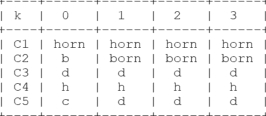}
\end{table}
We see that the final state depends on initial values and on distribution of them over $L_{1}$-branches. The result is also dependant on weight values, as even they are of the same degree of uncertainty. However, the final states are more or less equal in degrees of uncertainty in these cases (see also Sec. \ref{disc}). Moreover, these final states are stable: if we take them as initial ones, we obtain them at the end. Thus, the model corresponds well to the idea of the stationary state of the system under constant insolation and wind. Let us note one more time, that these natural factors are not changed by a transfer function in model iterations, unlike \cite{groump}. However, we can change them outwardly in order to compare results with the similar ones of \cite{groump}.

\section{Discussion}
\label{disc}
\subsection{Natural concepts' changing}
We consider only Case 3 here. Let insolation first be a constant --- $C_{1} = const$--- and wind $C_{3}$ increases: Table \ref{tab5}.
\begin{table}
\caption{Iterations of concept values, Case 3, $C_{1} = const$ and the $C_{3}$ increases}
\label{tab5}
\includegraphics[scale=1.1]{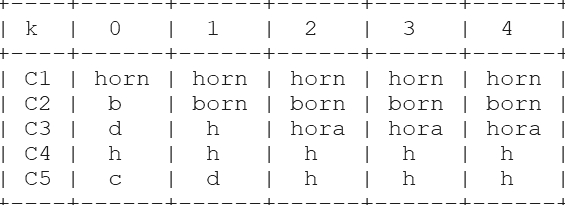}
\end{table}
We see the environment temperature $C_{2}$ has changed the uncertainty value but remains at the same $L_{1}$ level. The output of the photovoltaic subsystem $C_{4}$ as expected, has not changed, and the wind-turbine $C_{5}$ output has increased.

Let us decrease the insolation value else: Table \ref{tab6}.
\begin{table}
\caption{Iterations of concept values, Case 3, the $C_{1}$ decreases and the $C_{3}$ increases}
\label{tab6}
\includegraphics[scale=1.1]{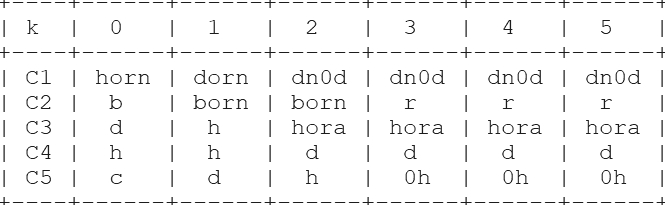}
\end{table}
We see the environment temperature $C_{2}$ has changed the uncertainty value but less than in the previous case. The output of the photovoltaic subsystem $C_{4}$ has slightly decreased, and the wind-turbine $C_{5}$ output has increased its $L_{1}$ level but less than in the previous case. This result is intuitionally clear; it corresponds with our weight definitions, and it is perhaps better than in \cite{groump} where the $C_{5}$ and the $C_{2}$ may be greater than here (with all the ambiguity  in  establishing the correspondence between numbers and lattice elements).

\subsection{Weights with negative values}
However, some negative number values are used in the weight matrix in \cite{groump}. The members in the sum (which define the map) with such weights deducted from the sum. We can also use a similar deduction with the help of (\ref{eqmin}). We mark such weight matrix members with the minus sign. Sets of generators of the lattice elements, including such weights in the join, will be deducted from the join by (\ref{eqmin}). Thus, we consider the following weight matrix corresponding to the similar one of \cite{groump} and Table \ref{tab1}: Table \ref{tab7}.
\begin{table}
\caption{Weights with negative elements in the cognitive map for Hybrid Energy Systems}
\label{tab7}
\begin{tabular}
{c|c|c|c|c|c}
 & $C_{1}$ & $C_{2}$ & $C_{3}$ & $C_{4}$ & $C_{5}$ \\ \hline
$C_{1}$ & Th & born & 0 & hora & b  \\ \hline
$C_{2}$ & 0 & 0 &  0 & - Tb & 0  \\ \hline   
$C_{3}$ & 0 & - ca0c & Th & bora & hn0h  \\ \hline   
$C_{4}$ & 0 & 0  & 0  & 0 &  0  \\ \hline
$C_{5}$ & 0 & 0  & 0  & 0 &  0
\end{tabular}
\end{table}

Let us consider again first that insolation is constant --- $C_{1} = const$--- and the wind $C_{3} = const$ is constant too: Table \ref{tab8}.
\begin{table}
\caption{Iterations of concept values with weights negative elements, Case 3, $C_{1} = const$  and $C_{3} = const$ }
\label{tab8}
\includegraphics[scale=1.1]{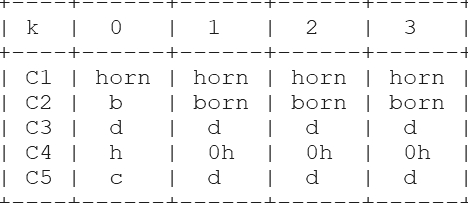}
\end{table}
This is almost Table \ref{tab4}. Only, the $C_{4}$ has slightly decreased its uncertainty value.

Let us increase the wind value: Table \ref{tab9}.
\begin{table}
\caption{Iterations of concept values with the weights' negative elements, Case 3, $C_{1} = const$ and the concept $C_{3}$ increases}
\label{tab9}
\includegraphics[scale=1.1]{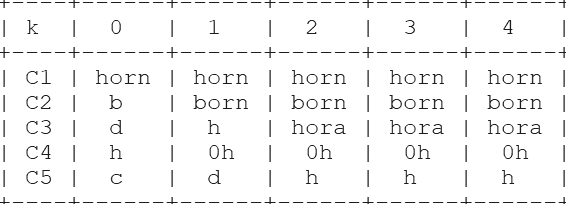}
\end{table}
Again, this is almost Table \ref{tab5}, and the $C_{4}$ has the same decrease as in the previous case.

Let us decrease else the insolation value: Table \ref{tab10}.
\begin{table}
\caption{Iterations of concept values with weights' negative elements, Case 3, the $C_{1}$ decreases and the $C_{3}$ increases}
\label{tab10}
\includegraphics[scale=1.1]{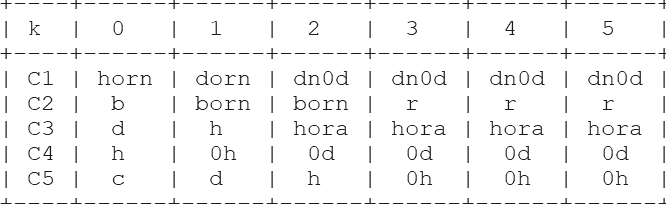}
\end{table}
This is almost Table \ref{tab6}. Only, the $C_{4}$ has slightly decreased its uncertainty value.
Thus, we see the results do not change qualitatively.

\subsection{The dependence on initial values }
However, what happens if we change, e.g., the wind initial value? Let it be in Case 3 (Table \ref{tab2}) $C_{5} = d$ instead of $C_{5} = c$. In the uncertainty sense, such a replacement changes almost nothing intuitively. The results are in Table \ref{tab11}.
\begin{table}
\caption{Iterations of concept values  with weights' negative elements, Case 3, the $C_{1}$ decreases and the $C_{3}$ increases, initial $C_{5} = d$}
\label{tab11}
\includegraphics[scale=1.1]{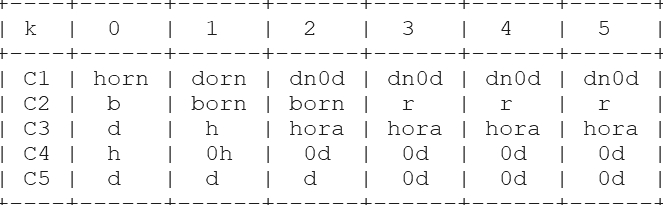}
\end{table}
We see that wind-turbine output even decreases, though, the wind value increases. All other values are the same. Removing minus from the weight matrix elements only replaces the $C_{4}$ uncertainty value to the same in Table \ref{tab6}. Thus, initial values can influence the modelled system behaviour crucially in general.

The thing is that different sets of generators in initial values may really lead to different final values, since the lattice used is not linearly ordered, and all $C_{i}$'s lie inside the two initial ones: $\forall k: \;C_{i}^{k}\leqslant C_{i}^{0}\vee C_{i}^{1}$. In our case, $\forall k$ $C_{5}^{k}\leqslant c\vee d = h$ in Table \ref{tab10}, and $C_{5}^{k}\leqslant d$ in Table \ref{tab11}; $C_{5}^{1} = d$ in both the cases, it does not depend on $C_{5}^{0}$.
Thus, different initial values may lead to different system stable states. However, such a dependence may be excluded by the learning process of weight elements, if we know the demanding output range (see Sec. \ref{learn}). We consider these calculations in the following subsection.

Also, such effects of intuitive contradiction can be indirectly related to our lattice determination: we went from the bi-lattice $L_{2}\times L_{1}$ to the lattice $L$ where the elements $0c$ and $0d$ are the same generators as $b, ba0$, and $bn0$. Though, in the lattice $L_{1}$, such values are more significant than the bottom level (the level of $b, ba0$, and $bn0$ in the $L$). Hence, our real partial order is generated by the number of generators of lattice elements, and it may be different from the intuitive interpretation of $L_{2}\times L_{1}$ partial orders.

\subsection{Learning}
We use here the algorithm of Sec. \ref{learn} which trains weights so that the output concepts would be inside the desired lattice $L$ subsets. We use formula (\ref{eq19}) in the following form in our case:
\begin{equation}\label{eq19n}
\triangle w_{ji}^{k-1} = A_{j}^{k-1}\Rightarrow (f_{i}^{k-1}\Rightarrow gen_{ri}),
\end{equation}
since, both residuals become the lattice implication when the monoid multiplication is the meet.

We determine desired output concept sets for the $C_{4}$ and the $C_{5}$ as:
\begin{gather}
Doc_{4} = \{0d, d, da0d, dn0d\} \\
Doc_{5} = \{0h, h, ha0h, hn0h\}.
\end{gather}
First, we apply the learning algorithm at the end of firing the map (Case 3, initial  $C_{5} = d$). Then we obtain the weight matrix in Table \ref{tab12} instead of Table \ref{tab7}.
\begin{table}
\caption{Weight matrix obtained from the learning process of the $C_{5}$ at the end of iterations, Case 3, the $C_{1}$ decreases and the $C_{3}$ increases, initial $C_{5} = d$}
\label{tab12}
\includegraphics[scale=1.2]{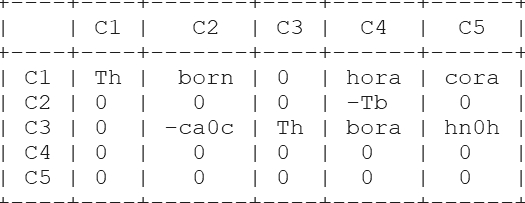}
\end{table}

In this case, we obtain the following iterations of concept values when the weight matrix is in Table \ref{tab12} for initial $C_{5} = c$ and $C_{5} = d$, and the comparison is made with the first elements of $Doc_{4}, Doc_{5}$: Table \ref{tab13}. We see that final output concept values do not depend on the initial ones now.
\begin{table}
\caption{Iterations of concept values  with weights from Table \ref{tab12}, Case 3, the $C_{1}$ decreases and the $C_{3}$ increases, initial $C_{5} = c$ and $C_{5} = d$.}
\label{tab13}
\includegraphics[scale=1.1]{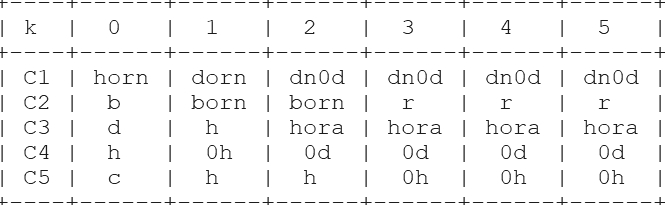}
\includegraphics[scale=1.1]{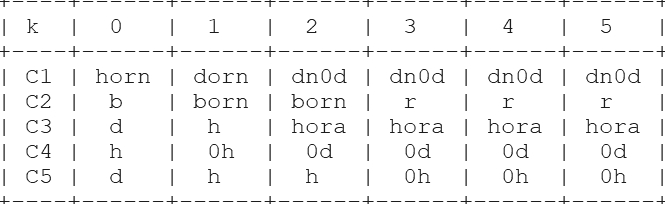}
\end{table}

If we apply learning at each iteration step, we will obtain different weights: Table \ref{tab14}. We see that matrix elements for $C_{4}$ calculation are also changed in this case. We have seen, though, that it is not needed in reality: the $C_{4}$ process converges in the desired region without learning.
In Table \ref{tab15}, the output concept results are also obtained with the learning process. We see, that the process leads to different final values of the $C_{5}$ for initial ones $c$ and $d$, though, both of them are inside the desired output set.

If we use the learned matrix of Table \ref{tab14} in the map firing, we will obtain the output concept result for the initial $C_{5} = c$, which matches the same one for the initial $C_{5} = d$ and differs from the similar one in Table \ref{tab15}: Table \ref{tab16}.
\begin{table}
\caption{Weight matrix obtained from the learning process of the $C_{5}$ at each step of the iterations, Case 3, the $C_{1}$ decreases and the $C_{3}$ increases, initial $C_{5} = d$}
\label{tab14}
\includegraphics[scale=1.1]{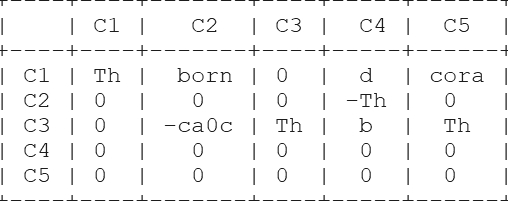}
\end{table}
\begin{table}
\caption{Iterations of concept values  with weights from Table \ref{tab14}, Case 3, the $C_{1}$ decreases and the $C_{3}$ increases, initial $C_{5} = c$ and $C_{5} = d$. }
\label{tab15}
\includegraphics[scale=1.1]{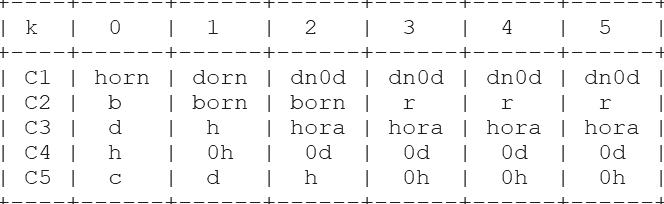}
\includegraphics[scale=1.1]{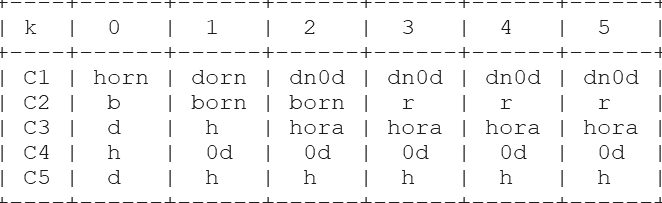}
\end{table}
\begin{table}
\caption{Iterations of concept values  without learning with weights from Table \ref{tab14}, Case 3, the $C_{1}$ decreases and the $C_{3}$ increases, initial $C_{5} = c$ and $C_{5} = d$.}
\label{tab16}
\includegraphics[scale=1.1]{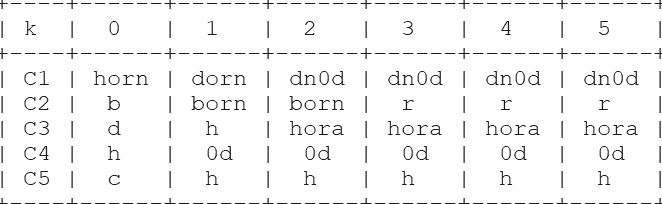}
\includegraphics[scale=1.1]{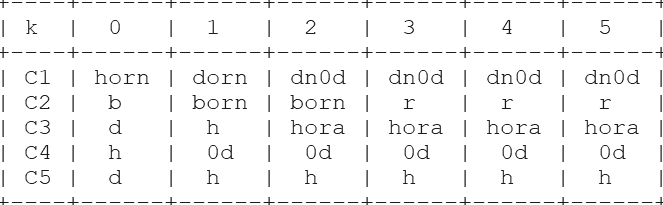}
\end{table}

We see that the output concepts converge in the desired range in all these variants, and the output values do not depend on whether the initial one is $c$ or $d$ (if we calculate them with the learned weight matrix), unlike the previous Subsection.

\subsection{Runtime}
Finally, all the calculations of implications in (\ref{eq11d}) were performed by the quick algorithm of \cite{maximov20e} and the resulting timing is depicted in Fig. \ref{fig5}.
\begin{figure}
  \includegraphics[scale=0.7]{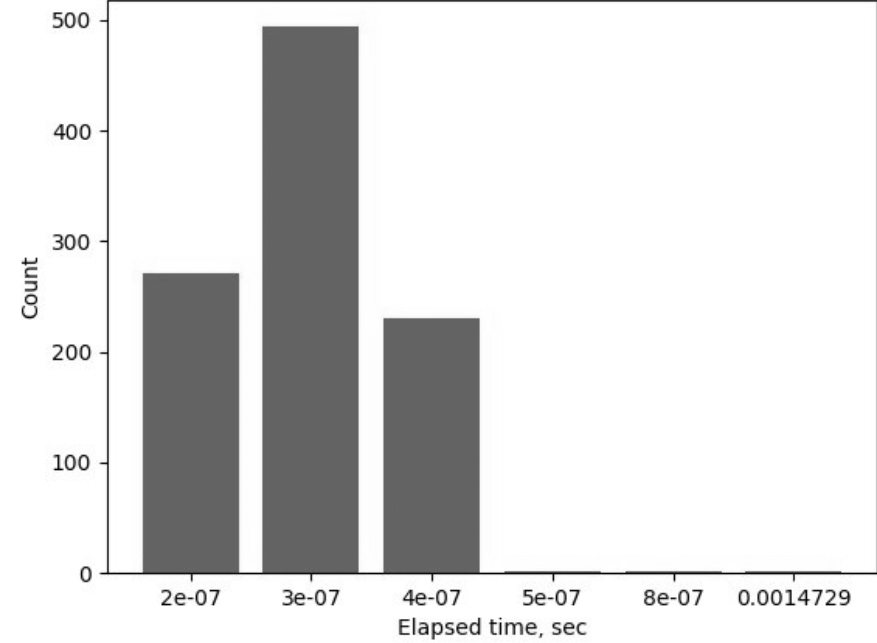}
\caption{Calculation elapsed time}
\label{fig5}       
\end{figure}

\section{Conclusion}
\label{concl}
We have considered the concept of cognitive maps in which all weights and data take values in a partially-ordered set (exactly, in a lattice) of linguistic quantities. Thus, experts get a wider scale for their linguistic assessments than in a fuzzy case. Such maps converge under some limitations on the set of map variable values. We give also the algorithm to learn the map weight matrix in order to select values so that they are in the desired range of the lattice.

We give a detailed consideration of a modelling example versus using a fuzzy cognitive map. We obtain even more realistic results, since, in our approach, immutable or externally modifiable concepts do not change by the map recount. In ordinary fuzzy cognitive maps, such concepts are automatically changed by a transfer (trimming) function.

Thus, it seems the consideration of  multi-valued cognitive maps in the paper demonstrates the self-consistency and correspondence to the reality of the approach, despite some ambiguity in interpretation of linguistic assessments. Moreover, the approach also provides more opportunities than in fuzzy maps for expert evaluations.

We have left for future investigations the problem of the comparison of different expert opinions. You need to introduce some conception of multi-valued numbers as the lattice subsets in order to do this. Also, we have left for the future the investigation of variants to define a universal residuated construction for the lattice used as a scale of linguistic values of map variables.


\section*{References}

  \bibliographystyle{elsarticle-num}
  \bibliography{Maximov_Cognitive}





\end{document}